\documentclass[10pt]{article} % For LaTeX2e
\usepackage[accepted]{tmlr}
\usepackage{amsmath,amsfonts,bm}

\def\1{\bm{1}}

\DeclareMathAlphabet{\mathsfit}{\encodingdefault}{\sfdefault}{m}{sl}
\SetMathAlphabet{\mathsfit}{bold}{\encodingdefault}{\sfdefault}{bx}{n}

\usepackage{graphicx}
\usepackage{amsmath,amsthm,amssymb}
\usepackage{hyperref}
\usepackage{url}

\newcommand{\X}{{\mathcal{X}}}
\newcommand{\Y}{{\mathcal{Y}}}

\newcommand{\dist}{d_{\mathcal{X}}}

\newtheorem{lemma}{Lemma}

\newtheorem{corollary}{Corollary}

\newtheorem{assumption}{Assumption}

\title{Fast Computation of Leave-One-Out Cross-Validation\\ for $k$-NN Regression}

\author{\name Motonobu Kanagawa \email motonobu.kanagawa@eurecom.fr\\
      \addr Data Science Department \\
      EURECOM
}

\def\month{12}  % Insert correct month for camera-ready version
\def\year{2024} % Insert correct year for camera-ready version
\def\openreview{\url{https://openreview.net/forum?id=SBE2q9qwZj}} % Insert correct link to OpenReview for camera-ready version

\begin{document}

\maketitle

\begin{abstract}
We describe a fast computation method for leave-one-out cross-validation (LOOCV) for $k$-nearest neighbours ($k$-NN) regression.
We show that, under a tie-breaking condition for nearest neighbours, the LOOCV estimate of the mean square error for $k$-NN regression is identical to the mean square error of $(k+1)$-NN regression evaluated on the training data, multiplied by the scaling factor $(k+1)^2/k^2$. 
Therefore, to compute the LOOCV score, one only needs to fit $(k+1)$-NN regression only once, and does not need to repeat training-validation of $k$-NN regression for the number of training data. 
Numerical experiments confirm the validity of the fast computation method.
\end{abstract}

\section{Introduction}

$k$-Nearest Neighbours ($k$-NN) regression \citep{stone1977consistent} is a classic nonparametric regression method that often performs surprisingly well in practice despite its simplicity \citep{chen2018explaining} and thus has been actively studied both theoretically and methdologically \citep[e.g,][]{gyorfi2002distribution,kpotufe2011k,jiang2019non,azadkia2019optimal,madrid2020adaptive,kpotufe2021marginal,lalande2023numerical,ignatiadis2023empirical,matabuena2024knn}.
It has a wide range of applications such as missing-value imputation \citep{troyanskaya2001missing}, conditional independence testing \citep{runge2018conditional}, outlier detection \citep{breunig2000lof}, approximate Bayesian computation \citep{biau2015new}, and function-valued regression \citep{lian2011convergence}, to just name a few.

For any test input, $k$-NN regression obtains its $k$-nearest training inputs and averages the corresponding $k$ training outputs to predict the test output. 
Thus, the number $k$ of nearest neighbours (and the distance function on the input space) is a key hyperparameter of $k$-NN regression and must be selected carefully.  
Indeed, theoretically, it is known that  $k$ should increase as the training data size $n$ increases for $k$-NN regression to converge to the true function \citep[e.g.,][Theorem 6.2]{gyorfi2002distribution}.
Hence, one should not use a prespecified value for $k$ (e.g., $k=5$) and must select $k$ depending on the training data.

A standard way for selecting the hyperparameters of a learning method is cross-validation \citep[e.g.,][Section 7.10]{hastie2009elements}.
However, cross-validation can be costly when the training data size is large. 
In particular, leave-one-out cross-validation (LOOCV) \citep{stone1974cross} can be computationally intensive since, if naively applied, it requires training the learning method on a training dataset of size $n-1$ and repeating it $n$ times. 
This issue also applies to the use of LOOCV for $k$-NN regression. 
On the other hand, theoretically, LOOCV is known to be better than cross-validation of a smaller number of split-folds as an estimator of the generalization error \citep[e.g.,][Section 6.2]{arlot2010survey}. 
\citet{azadkia2019optimal} theoretically analyses LOOCV for $k$-NN regression in selecting $k$ and discusses its optimality. 

This paper describes a fast method for computing LOOCV for $k$-NN regression. 
Specifically, we show that, under a tie-breaking assumption for nearest neighbours, the LOOCV estimate of the mean-squared error is identical to {\em the mean square error of $(k+1)$-NN regression evaluated on the training data, multiplied by the scaling factor $(k+1)^2 / k^2$} (Corollary~\ref{coro:LOOCV-formula} in Section~\ref{sec:method}).
Therefore, to perform LOOCV, one only needs to fit $(k+1)$-NN regression {\em only once} and does not need to repeat the $k$-NN fit $n$ times. 
To our knowledge, this method has not been reported in the literature. 

This paper is organised as follows. 
We describe $k$-NN regression and LOOCV in Section~\ref{sec:kNN-LOOCV-setting}.
We present the fast computation method in Section~\ref{sec:method}, empirically confirm its validity in Section~\ref{sec:experiment},
\textcolor{black}{discuss the tie-breaking condition in  Section~\ref{sec:discuss-tie-break}}, and conclude in Section~\ref{sec:conclusions}.

\section{$k$-NN Regression and LOOCV} 
\label{sec:kNN-LOOCV-setting}

Suppose we are given $n$ input-output pairs as training data:
\begin{equation} \label{eq:train-data-full}
    D_n := \{ (x_1, y_1), \dots, (x_n, y_n) \} \subset \X \times \Y,
\end{equation}
where $\X$ and $\Y$ are the input and output spaces. 
Specifically, $\X$ is a metric space with a distance metric $\dist: \X \times \X \mapsto [0, \infty)$ (e.g., $\X = \mathbb{R}^D$ is a $D$-dimensional space and $\dist(x,x') = \| x - x'\|$ is the Euclidean distance). 
The output space $\Y$ can be discrete (e.g., $\Y = \{0,1\}$ in the case of binary classification) or continuous (e.g., $\Y = \mathbb{R}$ for the case of real-valued regression, or $\Y = \mathbb{R}^M$ for vector-valued regression with $M$ outputs). 
Below, we use the following notation for the set of training inputs:
\begin{equation} \label{eq:imput-data-set}
X_n := \{x_1, \dots, x_n \}.    
\end{equation}

Assuming that there is an unknown function $f: \X \to \Y$ such that 
$$
y_i = f(x_i) + \varepsilon_i, \quad i = 1,\dots, n,
$$
where $\varepsilon_1, \dots, \varepsilon_n \in \Y$ are independent zero-mean noise random variables, the task is to estimate the function $f$ based on the training dataset $D_n$. 
$k$-NN regression is a simple, nonparametric method for this purpose, which often performs surprisingly well in practice and has solid theoretical foundations \citep[e.g,][]{gyorfi2002distribution,kpotufe2011k,chen2018explaining}. 

\subsection{$k$-NN Regression}
Let $k \in \mathbb{N}$ be fixed. 
For an arbitrary test input $x^* \in \X$, $k$-NN regression predicts its output by first searching for the $k$-nearest neighbours of $x^*$ from $X_n$ and then computing the average of the corresponding $k$ training outputs. 
To be more precise, define the set of indices for $k$ nearest neighbours as
\begin{align} 
& {\rm NN}(x^*,\ k,\ X_n) := \{ \ i_1,  \dots, i_k \in \{1, \dots, n \} \ \mid \nonumber \\
& \quad  \dist( x^*, x_{i_1}) \leq \cdots \leq \dist( x^*, x_{i_k}) \leq    \dist( x^*, x_j) \ \text{for all } j \in \{1, \dots, n\} \backslash \{i_1, \dots, i_k \} \ \}.    \label{eq:NN-def}   
\end{align}
Then, the $k$-NN prediction for $x^*$ is defined as
$$
\hat{f}_{k, D_n}(x^*) := \frac{1}{k} \sum_{i \in {\rm NN}(x^*,\ k,\ X_n)} y_i.
$$

%Note that, as indicated in the definition in \eqref{eq:NN-def} of using $<$ instead of $\leq$, we assume that there is no tie in the distances for simplicity. 
In the following, we make the following tie-breaking assumption, which makes ${\rm NN}(x_\ell,\ k,\ X_n)$ uniquely specified for $\ell = 1,\dots,n$  (i.e., the inequalities `$\leq$' in \eqref{eq:NN-def} with $x^* = x_\ell$ become `$<$').
\begin{assumption}[\textcolor{black}{Tie-breaking}] \label{as:train-input-points}
For all $i, j = 1, \dots, n$ with $i \not= j$, we have $x_i \not= x_j$ and $d_\X(x_\ell, x_i) \not= d_\X(x_\ell, x_j)$ for all $\ell = 1,\dots, n$.
\end{assumption}

\textcolor{black}{
This tie-breaking condition is usually satisfied if $x_1, \dots, x_n$ are multivariate and sampled from a continuous distribution. Section~\ref{sec:discuss-tie-break} provides a further discussion.
}

\subsection{Leave-One-Out Cross-Validation (LOOCV)}

LOOCV for $k$-NN regression is defined as follows. 
Consider the case $\mathcal{Y} = \mathbb{R}^M$ with $M \in \mathbb{N}$ ($M=1$ is the case of standard real-valued regression).
For each $\ell = 1, \dots, n$, consider the training dataset~\eqref{eq:train-data-full} with the $\ell$-th pair $(x_\ell, y_\ell)$ removed: 
\begin{align*}
D_n \backslash \{ (x_\ell, y_\ell) \}  = \left\{ (x_1,y_1), \dots, (x_{\ell - 1}, y_{\ell - 1}), (x_{\ell + 1}, y_{\ell + 1}), \dots, (x_n, y_n)  \right\}.    
\end{align*}
Then, the $k$-NN prediction for any $x^*$ based on $D_n \backslash \{ (x_\ell, y_\ell) \}$ is 
$$
\hat{f}_{k, D_n \backslash \{ (x_\ell, y_\ell) \} } (x^*) =   \frac{1}{k} \sum_{i \in {\rm NN}(x^*,\ k,\ X_n \backslash \{ x_\ell \})} y_i.
$$
Then, the LOOCV score for $k$-NN regression can be defined as 
\begin{equation} \label{eq:LOOCV-naive}
    {\rm LOOCV}(k, D_n) := \frac{1}{n} \sum_{\ell = 1}^n \left\| \hat{f}_{k, D_n \backslash \{ (x_\ell, y_\ell) \} } (x_\ell) - y_\ell \right\|^2.
\end{equation}
That is, for each $\ell = 1,\dots, n$, the held-out pair $(x_\ell, y_\ell)$ is used as validation data for the $k$-NN regression fitted on the training data $D_n \backslash \{ (x_\ell, y_\ell) \}$ of size $n-1$. 
Calculating the LOOCV score for a large $n$ is computationally expensive since one needs to fit $k$-NN regression $n$ times if naively implemented. 
Next, we will show how this computation can be done efficiently by fitting $(k+1)$-NN regression only once.

\section{Fast Computation of LOOCV for $k$-NN Regression}
\label{sec:method}

To compute the LOOCV score~\eqref{eq:LOOCV-naive}, we need to compute, for each $\ell = 1,\dots, n$,
\begin{equation} \label{eq:kNN-LOOCV-on-xell}
    \hat{f}_{k, D_n \backslash \{ (x_\ell, y_\ell) \} } (x_\ell) =   \frac{1}{k} \sum_{i \in {\rm NN}(x_\ell,\ k,\ X_n \backslash \{ x_\ell \})} y_i.
\end{equation}
To this end, we need to obtain ${\rm NN}(x_\ell,\ k,\ X_n \backslash \{ x_\ell \})$, the indices for the $k$ nearest neighbours of $x_\ell$ in $X_n \backslash \{x_\ell\}$.

The key insight is the following simple fact: {\em The union of $\{x_\ell\}$ and the $k$ nearest neighbours of $x_\ell$ in $X_n \backslash \{x_\ell \}$ is identical to the $k+1$ nearest neighbours of $x_\ell$ in $X_n$.}  
That is, under Assumption~\ref{as:train-input-points}, we have
\begin{equation} \label{eq:identities-NN}
    {\rm NN}(x_\ell,\ k,\ X_n \backslash \{ x_\ell \}) \cup \{ x_\ell \} = {\rm NN}(x_\ell,\ k+1,\ X_n  ).
\end{equation}
See Figure~\ref{fig:illustration-knn} for an illustration of the case $k = 3$.
The reasoning is as follows.
The first nearest neighbour of $x_\ell$ in $X_n$ is, of course, $x_\ell \in X_n$. 
The second nearest neighbour of $x_\ell$ in $X_n$ is the first nearest neighbour of $x_\ell$ in $X_n \backslash \{ x_\ell \}$. 
Generally, for $1 \leq m \leq k$, the $(m+1)$-th nearest neighbour of $x_\ell$ in $X_{\ell}$ is the $m$-th nearest neighbour of $x_\ell$ in $X_n \backslash \{x_\ell\}$.
Therefore, the identity~\eqref{eq:identities-NN} holds.

\begin{figure}[t]
    \centering
    \includegraphics[width=0.4\linewidth]{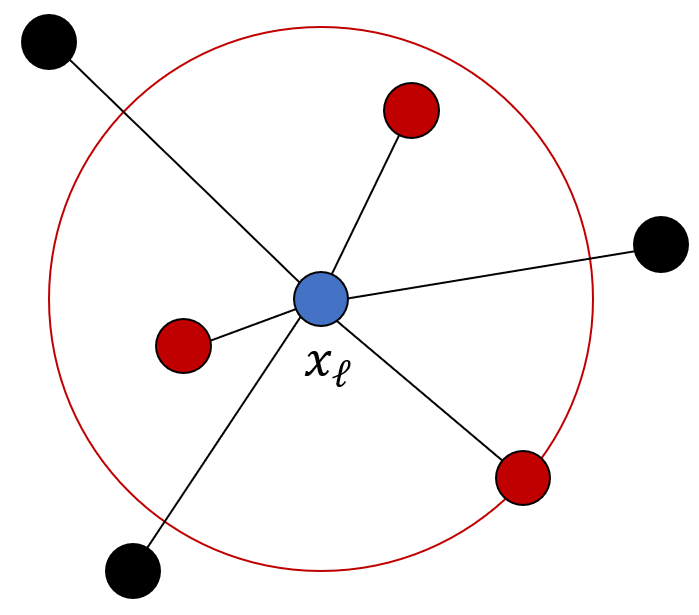}
    \caption{Illustration of $k=3$ nearest neighbours of $x_\ell$. The blue point represents $x_\ell$, the three red points are the $k=3$ nearest neighbours of $x_\ell$ in $X_n \backslash \{ x_\ell \}$, the black points are other points in $X_n \backslash \{ x_\ell \}$, and the red circle is the sphere of radius equal to the distance between $x_\ell$ and its third nearest neighbour in $X_n \backslash \{ x_\ell \}$ (the red point on the circle).}
    \label{fig:illustration-knn}
\end{figure}

Using~\eqref{eq:identities-NN}, we can rewrite \eqref{eq:kNN-LOOCV-on-xell} in terms of $(k+1)$-NN regression, as summarized as follows. 
\begin{lemma} \label{lemma:kNN-held-out-pred}
Under Assumption~\ref{as:train-input-points}, we have, for all $k \in \mathbb{N}$ and $\ell = 1, \dots, n$,  
$$
    \hat{f}_{k, D_n \backslash \{ (x_\ell, y_\ell) \} } (x_\ell) 
    = \frac{k+1}{k}  \hat{f}_{k+1, D_n} (x_\ell)  - \frac{1}{k} y_\ell
$$
\end{lemma}

\begin{proof}
Using~\eqref{eq:identities-NN}, we have 
\begin{align*}
\hat{f}_{k, D_n \backslash \{ (x_\ell, y_\ell) \} } (x_\ell) 
& =  \frac{1}{k} \sum_{i \in {\rm NN}(x_\ell,\ k,\ X_n \backslash \{ x_\ell \})} y_i   
=   \frac{1}{k} \left( \sum_{i \in {\rm NN}(x_\ell,\ k+1,\ X_n )} y_i - y_\ell \right) \\
& =  \frac{k+1}{k} \frac{1}{k+1}   \sum_{i \in {\rm NN}(x_\ell,\ k+1,\ X_n )} y_i - \frac{1}{k} y_\ell \\
& = \frac{k+1}{k}  \hat{f}_{k+1, D_n} (x_\ell)  - \frac{1}{k} y_\ell.  
\end{align*}
  
\end{proof}

%, which is fitted on the leave-one-out dataset $D_n \backslash \{ (x_\ell, y_\ell) \}$ and evaluated on the held-out input $x_\ell$, 

Lemma~\ref{lemma:kNN-held-out-pred} shows that the leave-one-out $k$-NN prediction $\hat{f}_{k, D_n \backslash \{ (x_\ell, y_\ell) \} } (x_\ell)$ on the held-out input $x_\ell$ can be written as $\frac{k+1}{k}  \hat{f}_{k+1, D_n} (x_\ell)  - \frac{1}{k} y_\ell$, which does {\em not} involve the hold-out operation of removing $(x_\ell, y_\ell)$ from $D_n$; it just requires fitting the $(k+1)$-NN regression $\hat{f}_{k+1, D_n}$ on $D_n$, evaluate it on $x_\ell$, scale it by $(k+1)/k$ and subtract it by $y_\ell / k$. 

Thus, to compute the LOOCV score~\eqref{eq:LOOCV-naive}, one needs to fit the $(k+1)$-NN regression {\em only once} on the dataset $D_n$. 
The resulting computationally efficient formula for the LOOCV score is given below.

\begin{corollary} \label{coro:LOOCV-formula}
Under Assumption~\ref{as:train-input-points}, for the LOOCV score in \eqref{eq:LOOCV-naive}, we have
\begin{equation} \label{eq:LOOCV-fast-formula}
    {\rm LOOCV}(k, D_n) = \left(\frac{k+1}{k} \right)^2  \frac{1}{n}  \sum_{\ell = 1}^n \left\|    \hat{f}_{k+1, D_n} (x_\ell) - y_\ell \right\|^2.
\end{equation}
\end{corollary}

\begin{proof}
Using Lemma~\ref{lemma:kNN-held-out-pred}, we have
\begin{align*}
& {\rm LOOCV}(k, D_n) 
 = \frac{1}{n} \sum_{\ell = 1}^n \left\| \hat{f}_{k, D_n \backslash \{ (x_\ell, y_\ell) \} } (x_\ell) - y_\ell \right\|^2  \\
& = \frac{1}{n} \sum_{\ell = 1}^n \left\|  \frac{k+1}{k}  \hat{f}_{k+1, D_n} (x_\ell)  - \frac{1}{k} y_\ell  - y_\ell \right\|^2 
 = \left(\frac{k+1}{k} \right)^2  \frac{1}{n} \sum_{\ell = 1}^n \left\|    \hat{f}_{k+1, D_n} (x_\ell) - y_\ell \right\|^2. 
\end{align*} 

\end{proof}

The expression~\ref{eq:LOOCV-fast-formula} shows that the LOOCV score~\eqref{eq:LOOCV-naive} for $k$-NN regression is simply the {\em residual sum of squares} (or mean-square error) of the $(k+1)$-NN regression fitted and evaluated on $D_n$, multiplied by the scaling factor $\left(\frac{k+1}{k} \right)^2$.
This scaling factor becomes large when $k$ is small, penalising overfitting and preventing too small $k$ from being chosen by LOOCV.

\section{Experiments} \label{sec:experiment}

We empirically check the validity of the formula~\eqref{eq:LOOCV-fast-formula} for efficient LOOCV computation.\footnote{The code for reproducing the experiments is available on {\tt https://github.com/motonobuk/LOOCV-kNN}}
%\footnote{The code for reproducing the experiments is available in the supplementary material.}
We consider a real-valued regression problem where $\X = \mathbb{R}^d$ and $\Y = \mathbb{R}$, \textcolor{black}{using two real datasets from {\tt scikit-learn}: ``Diabetes''   and ``Wine''.} 
We standardized each input feature to have mean zero and unit variance. 

We compare two approaches: one is the brute-force computation of the LOOCV score~\eqref{eq:LOOCV-naive} (``LOOCV-Brute''), and the other is the efficient computation based on the derived formula~\eqref{eq:LOOCV-fast-formula} (``LOOCV-Efficient''). 
We use the implementation of {\tt scikit-learn}\footnote{{\tt https://scikit-learn.org/stable/modules/generated/sklearn.neighbours.KneighboursRegressor.html}} with {\tt `\verb|kd_tree|'} for computing nearest neighbours.

\textcolor{black}{
 Figure~\ref{fig:experiment-results} (Diabetes) and Figure~\ref{fig:experiment-results-wine} (Wine) show the respective results.}
On the left of each figure, we show the LOOCV scores computed by the two methods for different values of $k$, the number of nearest neighbours. 
They exactly coincide; we also have checked this numerically. 
This result verifies the correctness of the formula~\eqref{eq:LOOCV-fast-formula}.  
On the right, we show the computation times\footnote{CPU: 1.1 GHz Quad-Core Intel Core i5.  Memory: 8 GB 3733 MHz LPDDR4X.}  required for either method for different training data sizes $n$ for fixed $k=5$.
While LOOCV-Brute's computation time increases linearly with the sample size $n$, LOOCV-Efficient's computation time remains almost unchanged and is negligible compared with LOOCV-Brute's computation time. 
This shows the effectiveness of the formula~\eqref{eq:LOOCV-fast-formula} for the fast computation of the LOOCV score.

\begin{figure}[t]
    \centering
    \includegraphics[width=0.45\linewidth]{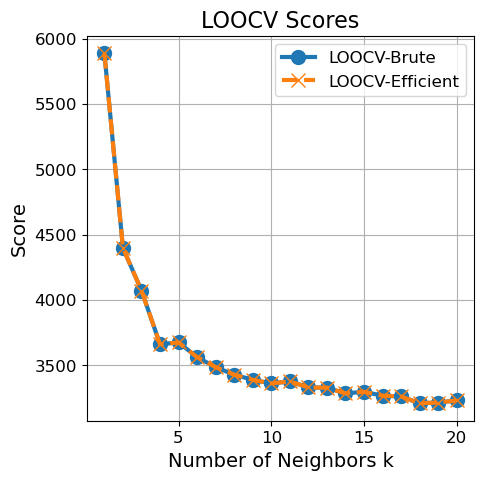} \quad 
    \includegraphics[width=0.45\linewidth]{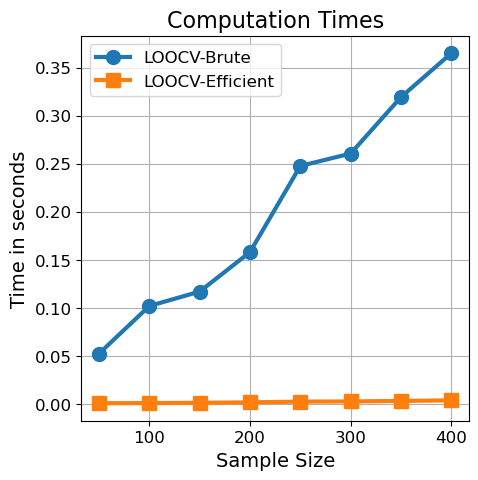}
    \caption{Experimental results on the Diabetes dataset. The left figure shows the LOOCV scores~\eqref{eq:LOOCV-naive} computed in the brute-force manner (``LOOCV-Brute'') and by using the derived formula~\eqref{eq:LOOCV-fast-formula} for different values of $k$. The right figure shows the computation times of either approach for different data sizes $n$ for fixed $k=5$.} 
    \label{fig:experiment-results}
\end{figure}

\begin{figure}[t]
    \centering
    \includegraphics[width=0.45\linewidth]{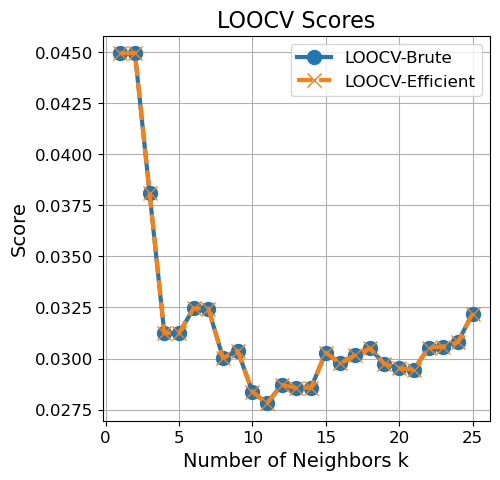} \quad
    \includegraphics[width=0.45\linewidth]{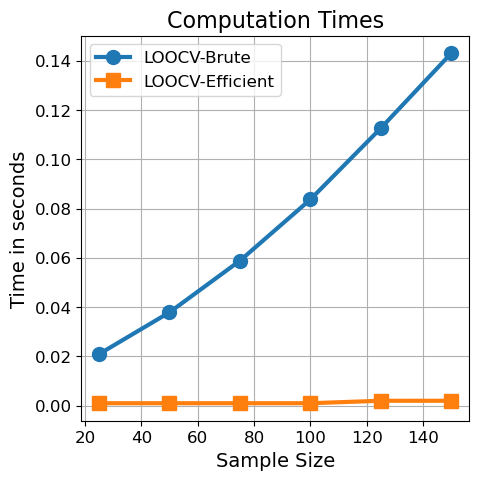}
    \caption{
    \textcolor{black}{
    Experimental results on the Wine dataset. The left figure shows the LOOCV scores~\eqref{eq:LOOCV-naive} computed in the brute-force manner (``LOOCV-Brute'') and by using the derived formula~\eqref{eq:LOOCV-fast-formula} for different values of $k$. The right figure shows the computation times of either approach for different data sizes $n$ for fixed $k=5$.}} 
    \label{fig:experiment-results-wine}
\end{figure}

\section{Discussion on the Tie-breaking Condition} 
\label{sec:discuss-tie-break}

\textcolor{black}{
Lastly, as a caveat, we consider the case where the tie-break condition in Assumption~\ref{as:train-input-points} is not satisfied.
We use the same Diabetes and Wine datasets as in Section~\ref{sec:experiment}, but only employ one input feature in each dataset: ``BMI'' for the Diabetes and ``malic-acid'' for the Wine.
Since each feature has many duplicates, there are many $i \not= j$ with $x_i = x_j$, thus violating the tie-breaking condition.  
Figure~\ref{fig:experiment-results-duplicate} shows the results corresponding to the left figures of Figures~\ref{fig:experiment-results} and \ref{fig:experiment-results-wine}.
LOOCV-Efficient is no longer exact and overestimates the LOOCV-Brute when $k$ is small.}

This phenomenon can be understood by considering the case where $k = 1$. 
Suppose there is $\ell \in \{ 1, \dots, n \}$ such that $x_\ell = x_i = x_j$ for some $i \not= j \not= \ell$.
In this case, the first nearest neighbour of $x_\ell$ in $X_n = \{x_1, \dots, x_n\}$ is $x_\ell$, $x_i$ and $x_j$, since $d_\X(x_\ell, x_\ell) = d_\X(x_\ell, x_i) = d_\X(x_\ell, x_j) = 0$. 
Therefore, depending on the tie-breaking rule of the nearest neighbour search algorithm, $x_i$ and $x_j$ may be selected as the first $k+1 = 2$ nearest neighbours of $x_\ell$ in $X_n$; in this case the identity in \eqref{eq:identities-NN} does not hold.

\textcolor{black}{
However, as shown in Figure~\ref{fig:experiment-results-duplicate}, LOOCV-Efficient overestimates LOOCV-Brute only for {\em small} $k$, where the LOOCV scores are {\em large}. 
Therefore, this issue would not severely affect the {\em best} $k$ that minimizes the LOOCV score.
Indeed, for the Diabetes data, the best $k$ selected by LOOCV-Efficient is $k = 17$ and identical to LOOCV-Brute.
For the Wine data, the best $k$ of LOOCV-Efficient is $k = 17$, while that of LOOCV-Brute is $k = 21$, which are similar and would not significantly change the prediction performance. 
}

The above example suggests that Assumption~\ref{as:train-input-points} is essential for the formula~\eqref{eq:LOOCV-fast-formula} to be exact. 
In practice, one could check whether Assumption~\ref{as:train-input-points} is satisfied by, e.g., checking whether there are no duplicates in the inputs $x_1, \dots, x_n$ beforehand; if they exist, one could resolve the duplicates by, e.g., taking the average of the outputs of duplicated samples.
However, as demonstrated in the results in Figures~\ref{fig:experiment-results} and \ref{fig:experiment-results-wine} for the full Diabetes dataset, Assumption~\ref{as:train-input-points} would be satisfied for many practical situations where the input features are multivariate and continuous.

\begin{figure}[t]
    \centering
    \includegraphics[width=0.45\linewidth]{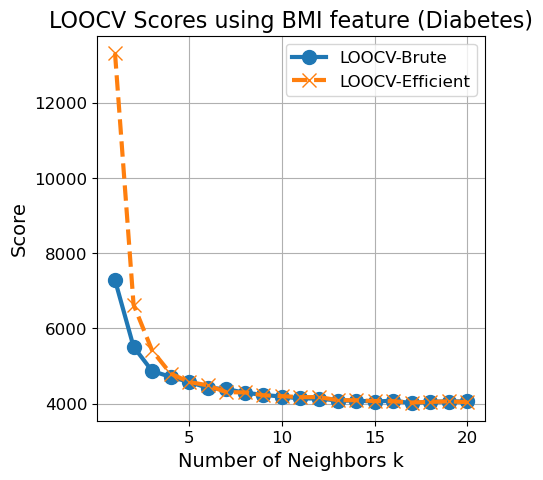}~~~~
    \includegraphics[width=0.45\linewidth]{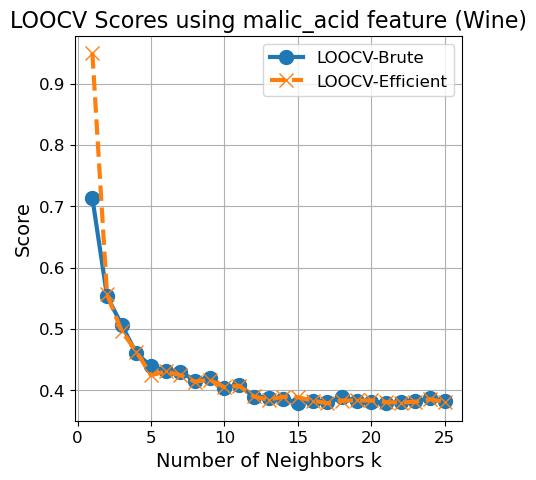}
    \caption{
    \textcolor{black}{
    LOOCV scores for the Diabetes dataset (left) and the Wine dataset (right), each of which only uses {\em one input feature}. 
    The used input feature has many duplicates and thus does not satisfy the tie-breaking condition in Assumption~\ref{as:train-input-points}.
    {\bf Left:} The best $k$ with the lowest LOOCV score is 17 for both LOOCV-Brute and LOOCV-Efficient.
    {\bf Right:} The best $k$ with the lowest LOOCV score is 21 for LOOCV-Brute and 17 for LOOCV-Efficient.}
    }
    \label{fig:experiment-results-duplicate}
\end{figure}

\section{Conclusions} \label{sec:conclusions}
We showed that LOOCV for $k$-NN can be computed quickly by fitting $(k+1)$-NN regression only once, evaluating the mean-square error on the training data and multiplying it by the scaling factor $(k+1)^2 / k^2$. 
By applying this technique, many applications of $k$-NN regression can be accelerated. 
It also opens up the possibility of using LOOCV for optimising not only $k$ but also the distance function $d_\X$; this would be one important future direction.

\subsection*{Acknowledgements} 
The author thanks Takafumi Kajihara for a discussion, as well as the Action Editor and the anonymous reviewers for their time and comments.

\bibliographystyle{apalike}
\bibliography{bibfile}

\end{document}